\documentclass[a4paper,11pt]{article}
\usepackage{amsthm,amsmath,amsfonts,natbib,amssymb,graphicx,color,times, algorithm, algorithmic, array}

\newcommand{\G}{\mathcal{G}}

\newcommand{\dG}{d^{\scriptscriptstyle G}}

\newcommand{\etaG}{\eta^{\scriptscriptstyle G}}

\newcommand{\R}[1]{\mathbb{R}^{#1}}
\newcommand{\RR}[2]{\mathbb{R}^{#1 \times #2}}

\newcommand{\NormUn}[1]{\left\|#1\right\|_1}
\newcommand{\NormDeux}[1]{\left\|#1\right\|_2}
\newcommand{\NormAlpha}[1]{\left\|#1\right\|_{\alpha}}

\newcommand{\NormFro}[1]{\left\|#1\right\|_{\mathrm{F}}}
\newcommand{\NormOmega}[2]{\sum_{G \in #1}\NormDeux{\dG \circ #2}}

\newcommand{\Diag}[1]{\mathrm{Diag} \! \left(#1\right)}

\newcommand{\Support}[1]{\mathrm{supp} \! \left(#1\right)}

\newcommand{\IntegerSet}[1]{\{1,\dots,#1\}}

\newtheorem{thm}{Theorem}[section]

\newtheorem{lem}[thm]{Lemma}

\newcommand{\MyEq}[1]{Eq.~(\ref{#1})}
\newcommand{\MyLem}[1]{Lemma~(\ref{#1})}
\newcommand{\MyFig}[1]{Fig.~\ref{#1}}
\newcommand{\MySec}[1]{Section~\ref{#1}}
\newcommand{\MyAlg}[1]{Algorithm~\ref{#1}}

\DeclareMathOperator*{\argmin}{arg\,min}

\author{
\begin{tabular}{lcr}
 & & \\ 
\textbf{Rodolphe Jenatton$^{1}$}     &    &  rodolphe.jenatton@inria.fr \\
\textbf{Guillaume Obozinski$^{1}$}   &    &  guillaume.obozinski@inria.fr \\
\textbf{Francis Bach$^{1}$}          &    &  francis.bach@inria.fr \\
\multicolumn{3}{l}{\small \textit{$^1$INRIA - WILLOW Project-team,}} \\
\multicolumn{3}{l}{\small \textit{Laboratoire d'Informatique de l'Ecole Normale Sup\'erieure (INRIA/ENS/CNRS UMR 8548),}} \\
\multicolumn{3}{l}{\small \textit{23, avenue d'Italie, 75214 Paris. France}} \\
 & & 
\end{tabular}
}

\title{Structured Sparse Principal Component Analysis}

\begin{document}

\maketitle

\begin{abstract}

We present an extension of sparse PCA, or sparse dictionary learning,
where the sparsity patterns of all dictionary elements are structured and
constrained to belong to a prespecified set of shapes.
This \emph{structured sparse PCA} is based on a structured regularization recently introduced by
\cite{GrosLasso}.
While classical sparse priors only deal with \textit{cardinality}, 
the regularization we use encodes higher-order information about the data.
We propose an efficient and simple optimization procedure to solve this
problem.
Experiments with two practical tasks, face recognition and the study of
the dynamics of a protein complex, demonstrate the benefits of the
proposed structured approach over unstructured approaches.

\end{abstract}

\section{Introduction}
Principal component analysis (PCA) is an essential tool for data analysis
and unsupervised dimensionality reduction, whose goal is to find,
among  linear combinations of the data variables, a sequence of orthogonal
factors that most efficiently explain the variance of the observations.

One of its main shortcomings is that, even if PCA finds a small number of
important factors, the factor themselves typically involve all original
variables.
In the last decade, several alternatives to PCA which find sparse and potentially interpretable factors have been proposed, notably non-negative matrix factorization (NMF)
\cite{NMF} and sparse PCA (SPCA) \cite{jolliffe2003mpc,zou2006spc,Witten}.

However, in many applications, only constraining the size of the factors
does not seem appropriate because the considered factors are not only expected
to be sparse but also to have a certain structure.
In fact, the popularity of NMF for face image analysis owes essentially to
the fact that the method happens to retrieve sets of variables that are
localized on the face and capture some features or parts of the face which
seem intuitively meaningful given our a priori. We might therefore gain in
the quality of the factors induced by enforcing directly this a priori in
the matrix factorization constraints. More generally, it is desirable to
encode higher-order information about the supports that reflects the
\textit{structure} of the data. For example, in computer vision, features
associated to the pixels of an image are naturally organized on a grid and
the supports of factors explaining the variability of images could be
expected to be localized, connected or have some other regularity with
respect to the grid. Similarly, in genomics, factors explaining the gene
expression patterns observed on a microarray could be expected to involve
groups of genes corresponding to biological pathways or set of genes that
are neighbors in a protein-protein interaction network.

Recent research on structured sparsity \cite{LaurentGuillaumeGroupLasso,
huang2009, GrosLasso} has highlighted the benefit of exploiting such
structure for variable selection and prediction in the context of
regression and classification.
In particular, \cite{GrosLasso} shows that, given any intersection-closed family of
patterns $\mathcal{P}$ of variables, such as all the rectangles on a
2-dimensional grid of variables, it is possible to build an ad hoc
regularization norm $\Omega$ that enforces that the support of the
solution of the least-squares regression regularized by $\Omega$ belongs
to the family~$\mathcal{P}$.

Capitalizing on these results, we aim in this paper to go beyond sparse PCA and
propose \textit{structured sparse PCA} (SSPCA), which
explains the variance of the data by factors that are not only sparse but
also respect some a priori structural constraints deemed relevant to model
the data at hand. We show how slight variants of the regularization term
of \cite{GrosLasso} can be used successfully to yield a structured and
sparse formulation of principal component analysis for which we propose a
simple and efficient optimization scheme.

The rest of the paper is organized as follows:
\MySec{sec:problem_statement} introduces the SSPCA problem in the
dictionary learning framework, summarizes the regularization considered in
\cite{GrosLasso} and its essential properties, and presents some simple
variants which are more effective in the context of PCA.
\MySec{sec:optimization} is dedicated to our optimization scheme for
solving SSPCA.
Our experiments in \MySec{sec:experiments} illustrate the benefits of our
approach through applications to face recognition and the study of the
dynamics of protein complexes.

\paragraph{Notation:} For any vector $y$ in $\R{p}$ and any $\alpha\!>\!
0$, we denote by $\NormAlpha{y}=(\sum_{j=1}^p |y_j|^\alpha)^{1/\alpha}$
the (quasi-)norm $\ell_\alpha$ of $y$.
Similarly, for any rectangular matrix $Y \in \RR{n}{p}$, we denote by
$\NormFro{Y}=(\sum_{i=1}^n\sum_{j=1}^p Y_{ij}^2)^{1/2}$ its Frobenius
norm, where $Y_{ij}$ is the $(i,j)$-th element of $Y$.
We write $Y^j$ for the $j$-th column of $Y$.
Given  $w$ in $\R{p}$ and a subset $J$ of $\IntegerSet{p}$, $w_J$ denotes
the vector in $\R{p}$ that has the same entries $w_j$ as $w$ for $j \in
J$, and null entries outside of $J$.
In addition, $\Support{w}=\{j\in\IntegerSet{p}\, ;\, w_j \neq 0 \}$ is
referred to as the \emph{support}, or \emph{nonzero pattern} of the vector
$w \in \R{p}$.
For any finite set $A$ with cardinality $|A|$, we also define the
$|A|$-tuple $(y^a)_{a \in A} \in \RR{p}{|A|}$ as the collection of
$p$-dimensional vectors $y^a$ indexed by the elements of $A$.
Furthermore, for two vectors $x$ and $y$ in $\R{p}$, we denote by $x \circ
y = (x_1y_1,\dots,x_p y_p)^\top \in \R{p}$ the elementwise product of $x$
and $y$.
Finally, we extend $\frac{a}{b}$ by continuity in zero with
$\frac{a}{0}=\infty$ if $a \neq 0$ and $0$ otherwise.

\section{Problem statement}\label{sec:problem_statement}

It is useful to distinguish two conceptually different interpretations of
PCA. In terms of \emph{analysis}, PCA sequentially projects the data on
subspaces that explain the largest fraction of the variance of the data. In
terms of \emph{synthesis}, PCA finds a basis, or orthogonal dictionary,
such that all signals observed admit decompositions with low
reconstruction error. These two interpretations recover the same basis of
principal components for PCA but lead to different formulations for \emph{sparse}
PCA. The \emph{analysis} interpretation leads to sequential formulations (\cite{spca2008aspremont,moghaddam2006sbs,jolliffe2003mpc})
that consider components one at a time and perform a \emph{deflation}  of the
covariance matrix at each step (see \cite{Mackey}). The \emph{synthesis}
interpretation leads to non-convex global formulations (\cite{zou2006spc,JulienOnLine,moghaddam2006sbs,lee}) which estimate
simultaneously all principal components, often drop the orthogonality
constraints, and are referred to as matrix factorization problems (\cite{singh2008uvm}) in
machine learning, and dictionary learning in signal processing.

The approach we propose fits more naturally in the framework of dictionnary learning, whose terminology we now introduce.

\subsection{Matrix factorization and dictionary learning}

 Given a matrix $X\in \RR{n}{p}$ of $n$ rows corresponding to $n$ observations in $\R{p}$, the dictionary learning problem is to find a matrix $V \in \RR{p}{r}$, called the \emph{dictionary}, such that each observation can be well approximated by a linear combination of the $r$ columns $(V^k)_{k \in \{1,\ldots,r\}}$ of $V$ called the \emph{dictionary elements}. If $U \in \RR{n}{r}$ is the matrix of the linear combination coefficients or \emph{decomposition coefficients}, the matrix product $UV^\top$ is called a decomposition of $X$.

Learning simultaneously the dictionary $V$ and the decomposition $U$ corresponds to a matrix factorization problem (see \cite{Witten} and reference therein). 
As formulated in \cite{bach2008csm} or \cite{Witten}, it is natural, when learning a decomposition, to penalize or constrain some norms or quasi-norms of $U$ and $V$, say  $\Omega_u$ and $\Omega_v$ respectively, to encode prior information---typically sparsity---about the decomposition of~$X$. This can be written generally as
\begin{equation}\label{eq:main_eq}
      \!\!\!\! \min_{  U \in \RR{n}{r},\, V \in \RR{p}{r}  }
                    \frac{1}{2np} \NormFro{ X \! - \! U V^\top \! }^2 + \lambda \sum_{k=1}^r \Omega_v(V^k) \qquad \mathrm{s.t.} \qquad  \forall k,\,\Omega_u(U^k) \, \leq \, 1,
\end{equation}
where the regularization parameter $\lambda \geq 0$ controls which extent the dictionary is regularized\footnote{From \cite{bach2008csm}, we know that our formulation is also equivalent to two unconstrained problems, with the penalizations
$\frac{\lambda}{2} \sum_{k=1}^r [\Omega_v(V^k)]^2 \! + \! [\Omega_u(U^k)]^2$ or $\lambda \sum_{k=1}^r \Omega_v(V^k) \Omega_u(U^k)$.  }.
If we assume that both regularizations $\Omega_u$ and $\Omega_v$ are convex, problem (\ref{eq:main_eq})  is convex w.r.t. $U$ for $V$ fixed and vice versa. It is however not \emph{jointly} convex in $(U,V)$.

The formulation of sparse PCA considered in \cite{lee} corresponds to a particular instance of this problem, where the dictionary elements are required to be sparse (without the orthogonality constraint $V^\top V = I$). This can be achieved by penalizing the columns of $V$ by a sparsity-inducing norm, e.g., the $\ell_1$ norm, $\Omega_v(V^k)=\NormUn{V^k}$.
In the next section we consider a regularization $\Omega_v$ which controls not only the sparsity but also the structure of the supports of dictionary elements.
\subsection{Structured sparsity-inducing norms}
The work of \cite{GrosLasso} considered a norm which induces structured sparsity in the following sense: the solutions to a learning problem regularized by this norm have a sparse support which moreover belongs to a certain set of groups of variables. Interesting sets of possible supports include set of variables forming rectangles when arranged on a grid and more generally convex subsets\footnote{We use the term \emph{convex} informally here. It can however be made precise with the notion of convex subgraphs (\cite{chung1997sgt}).}.

The framework of \cite{GrosLasso} can be summarized as follows:
if we denote by $\G$ a subset of the power set of $\IntegerSet{p}$, such that $ \bigcup_{G \in \G} G = \IntegerSet{p}$, we define a norm $\Omega$ on a vector $y \in \R{p}$ as

$$
\Omega(y) = \sum_{G\in\G} \bigg\{ \sum_{j\in G} (\dG_j)^2 |y_j|^2 \bigg\}^{\frac{1}{2}} = \NormOmega{\G}{y},
$$
where $(\dG)_{G\in\G} \in \RR{p}{|\G|}$ is a $|\G|$-tuple of $p$-dimensional vectors such that $\dG_j > 0$ if $j \in G$ and $\dG_j = 0$ otherwise.
This norm $\Omega$ linearly combines the $\ell_2$ norms of possibly overlapping groups of variables, with variables in each group being weighted by $(\dG)_{G\in\G}$. Note that a same variable $y_j$ belonging to two different groups $G_1, G_2 \in \G$ is allowed to be weighted differently in $G_1$ and $G_2$ (by respectively $d_j^{\scriptscriptstyle G_1}$ and $d_j^{\scriptscriptstyle G_2}$).

For specific choices of $\G$, $\Omega$ leads to standard sparsity-inducing norms. For example, when $\G$ is the set of all singletons, $\Omega$ is the usual $\ell_1$ norm (assuming that all the weights are equal to 1).

We focus on the case of a 2-dimensional grid where the set of groups $\G$ is the set of all horizontal and vertical half-spaces (see \MyFig{fig:rectangular_group_illustration} taken from \cite{GrosLasso}).
As proved in \cite[Theorem 3.1]{GrosLasso}, the $\ell_1/\ell_2$ norm $\Omega$ sets to zero some groups of variables $\NormDeux{  \dG \circ y }$, i.e., some entire horizontal and vertical half-spaces of the grid, and therefore induces rectangular nonzero patterns.
Note that a broader set of convex patterns can be obtained by adding in $\G$ half-planes with other orientations. In practice, we use planes with angles which are multiples of $\frac{\pi}{4}$.

\begin{figure}[!h]
    \begin{center}
    	 \includegraphics[scale=0.5]{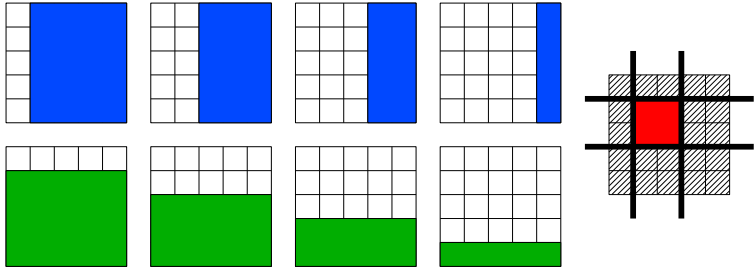}
	  \caption{(Left) The set of blue and green groups with their (not displayed) complements to penalize to select rectangles. (Right) In red, an example of recovered pattern in this setting.} 
    	\label{fig:rectangular_group_illustration}
    \end{center}
\end{figure}

Among sparsity inducing regularizations, $\ell_1$ is often privileged since it is convex. However, so-called concave penalizations, such as penalization by an $\ell_\alpha$ quasi-norm, which are closer to $\ell_0$ and penalize more aggressively small coefficients can be preferred, especially in a context where the unregularized problem, here dictionary learning is itself non convex.
In light of recent work showing the advantages of addressing sparse regression problems through concave penalization (e.g., see \cite{zhang2008msc}), we therefore generalize $\Omega$ to a family of non-convex regularizers as follows:
for $\alpha \in (0,1)$, we define the quasi-norm $\Omega^\alpha$ for all vectors $y \in \R{p}$ as
$$
	\Omega^\alpha(y) = \bigg\{ \sum_{G\in\G} \NormDeux{ \dG \circ y }^\alpha  \bigg\}^{\frac{1}{\alpha}}
			 = \NormAlpha{ \, (\NormDeux{ \dG \circ y })_{G \in \G} \, },
$$
where we denote by $(\NormDeux{ \dG \circ y })_{G \in \G} \in \RR{1}{|\G|}$ the $|\G|$-tuple composed of the different blocks $\NormDeux{ \dG \circ y }$. We thus replace the (convex) $\ell_1/\ell_2$ norm $\Omega$ by the (neither convex, nor concave) $\ell_\alpha/\ell_2$ quasi-norm $\Omega^\alpha$. Note that this modification impacts the sparsity induced at the level of groups, since we have replaced the convex  $\ell_1$ norm by the concave $\ell_\alpha$ quasi-norm.


\section{Optimization}\label{sec:optimization}

We consider the optimization of \MyEq{eq:main_eq} where we use $\Omega_v=\Omega^\alpha$ to regularize the dictionary $V$. We discuss in \MySec{sec:algorithm} which norms $\Omega_u$ we can handle in this optimization framework.

\subsection{Formulation as a sequence of convex problems}

We are now considering \MyEq{eq:main_eq} where we take $\Omega_v$ to be $\Omega^\alpha$, that is,
\begin{equation}\label{eq:main_eq_new_norm}
      \min_{  U \in \RR{n}{r},\, V \in \RR{p}{r}  }
              \frac{1}{2np} \NormFro{ X \! - \! U V^\top \! }^2 + \lambda \sum_{k=1}^r \Omega^\alpha(V^k) \qquad \mathrm{s.t.} \qquad  \forall k,\,\Omega_u(U^k) \, \leq \, 1.
\end{equation}
Although the minimization problem \MyEq{eq:main_eq_new_norm} is still convex in $U$ for $V$ fixed, the converse is not true anymore because of $\Omega^\alpha$. Indeed, the formulation in $V$ is non-differentiable and non-convex.
To address this problem, we use the variational equality based on the following lemma that is related\footnote{Note that we depart form \cite{bach2008cgl, micchelli2006lkf} who consider a quadratic upperbound on the \textit{squared} norm. We prefer to remain in the standard dictionary learning framework where the penalization is not squared.} to ideas from \cite{bach2008cgl, micchelli2006lkf}:

\begin{lem}\label{lem:EtaTrick}
Let $\alpha \in (0,2)$ and $\beta = \frac{\alpha}{2-\alpha}$. For any vector $y \in \R{p}$, we have the following equality
$$
 \NormAlpha{y} = \min_{ z  \in \mathbb{R}_+^p  } \,
                       \frac{1}{2}\sum_{j=1}^p \frac{y_j^2}{z_j} + \frac{1}{2} \|z\|_{\beta} ,
$$
and the minimum is uniquely attained for
$
  z_j = |y_j|^{2-\alpha}\NormAlpha{y}^{\alpha-1},\, \forall j \in \IntegerSet{p}.
$
\end{lem}
\begin{proof}
Let $\psi: \! z \! \mapsto \sum_{j=1}^p y_j^2 z_j^{-1} + \|z\|_{\beta}$ be the continuously differentiable function defined on $(0,+\infty)$.
We have
$ \lim_{\|z\|_{\beta}  \to \infty} \psi(z) \! = \! +\infty $
and
$ \lim_{z_j \to 0} \psi(z) \! = \! +\infty $ if $y_j \neq 0$
(for $y_j=0$, note that $\min_{z \geq 0} \psi(z) = \min_{z \geq 0, z_j = 0} \psi(z)$).
Thus, the infimum exists and it is attained.
Taking the derivative w.r.t. $z_j$ (for $z_j > 0$) leads to the expression of the unique minimum, expression that still holds for $z_j=0$.
\end{proof}
To reformulate problem~$(2)$, let us consider the $|\G|$-tuple $(\etaG)_{G \in \G} \in \RR{r}{|\G|}$ of $r$-dimensional vectors $\etaG$ that satisfy for all $k \in \IntegerSet{r}$ and $G \in \G$, $\etaG_k \geq 0$. It follows from \MyLem{lem:EtaTrick} that
$$
	2 \sum_{k=1}^r \Omega^\alpha(V^k) = \min_{ (\etaG)_{G \in \G} \in \RR{r}{|\G|}_+ } 
						\sum_{k=1}^r \bigg[ 
									\|(\etaG_k)_{G \in \G}\|_{\beta} + 
                                              				\sum_{G\in\G}  \NormDeux{V^k \circ \dG}^2 (\etaG_k)^{-1}
						             \bigg].
$$
If we introduce the matrix $\zeta \in \RR{p}{r}$ defined by\footnote{For the sake of clarity, we do not specify the dependence of $\zeta$ on $(\etaG)_{G \in \G}$.}
$
	\zeta_{jk} \!\! = \!\! \big\{ \sum_{  G \in \G,\ G \ni j  }
                             (\dG_j)^2 (\etaG_k)^{-1} \big\}^{-1}
$, we then obtain
$$
	2 \sum_{k=1}^r \Omega^\alpha(V^k) =  \!\!\! \min_{ (\etaG)_{G \in \G} \in \RR{r}{|\G|}_+ }
					              \sum_{k=1}^r  (V^k)^\top \Diag{\zeta^k}^{-1} \! V^k + \|(\etaG_k)_{G \in \G}\|_{\beta} .
$$
This leads to the following formulation
\begin{equation}\label{eq:main_eq_eta}
      \min_{  \substack{U,\ V,\ \Omega_u(U^k) \leq 1  \\  (\etaG)_{G \in \G} \in \RR{r}{|\G|}_+ } }
                \frac{1}{2np} \NormFro{ X \! - \! U V^\top \! }^2 + \frac{\lambda}{2} \sum_{k=1}^r \bigg[ (V^k)^\top \Diag{\zeta^k}^{-1} \! V^k
		+ \|(\etaG_k)_{G \in \G}\|_{\beta} \bigg],
\end{equation}
which is equivalent to \MyEq{eq:main_eq_new_norm} and convex with respect to $V$.

\subsection{Sharing structure among dictionary elements}\label{sec:shared_structure}

So far, the regularization quasi-norm $\Omega^\alpha$ has been used to induce a structure \emph{inside} each dictionary element taken separately.
Nonetheless, some applications may also benefit from a control of the structure \emph{across} dictionary elements.
 For instance it can be desirable to impose the constraint that $r$ dictionary elements share only a few different nonzero patterns.  In the context of face recognition, this could be relevant to model the variability of faces as the combined variability of several parts, with each part having a small support (such as eyes), and having its variance itself explained by \textit{several} dictionary elements (corresponding for example to the color of the eyes).
 
To this end, we consider $\mathcal{M}$, a partition of $\IntegerSet{r}$. Imposing that two dictionary elements $V^k$ and $V^{k'}$ share the same sparsity pattern is equivalent to imposing that $V^k_i$ and $V^{k'}_i$ are simultaneously zero or non-zero. Following the approach used for joint feature selection (\cite{OboTasJor09})
where the $\ell_1$ norm is composed with an $\ell_2$ norm, we compose the norm $\Omega^\alpha$ with the $\ell_2$ norm $V^M_i=\|(V^k_i)_{k \in M}\|_2$, of all $i^{\text{th}}$ entries of each dictionary element of a class $M$ of the partition, leading to the regularization:
\begin{equation}
\label{eq:constrained_eta}
  \sum_{M \in \mathcal{M}} \Omega^\alpha(V^M_i)=\sum_{M \in \mathcal{M}} \left[ \sum_{G\in\G} \NormDeux{(V^k_i \dG_i)_{i \in G,\, k \in M}}^{\alpha} \right]^{1/\alpha},
\end{equation}

  In fact, not surprisingly given that similar results hold for the group Lasso \cite{bach2008cgl}, it can be shown that the above extension is equivalent to the variational formulation 
\begin{eqnarray*}
      \min_{  \substack{U,\ V,\ \Omega_u(U^k) \leq 1  \\  (\etaG)_{G \in \G} \in \RR{|\mathcal{M}|}{|\G|}_+ } } \ \ 
             \!\!\! \frac{1}{2np} \NormFro{ X \! - \! U V^\top \! }^2 & \!\! + \!\! & \frac{\lambda}{2} \sum_{M \in \mathcal{M}} \bigg[  
             							\sum_{k \in M} (V^k)^\top \Diag{\zeta^M}^{-1} \! V^k + \|(\etaG_M)_{G \in \G}\|_{\beta}
             							\bigg]
\end{eqnarray*}
with class specific variables $\eta_M$, $\zeta^M$, $M \! \in \! \mathcal{M}$, defined by analogy with $\eta_k$ and $\zeta^k$, $k \! \in \! \IntegerSet{r}$.

\subsection{Algorithm}\label{sec:algorithm}

The main optimization procedure described in \MyAlg{alg:main_loop} is based on a cyclic optimization over the three variables involved, namely $(\etaG)_{G \in \G}$, $U$ and $V$.
We use \MyLem{lem:EtaTrick} to solve \MyEq{eq:main_eq_new_norm} by a sequence of problems that are convex in $U$  for fixed $V$ (and conversely, convex in $V$ for fixed $U$).
For this sequence of problems, we then present efficient optimization procedures based on block coordinate descent (BCD) \cite[Section 2.7]{bertsekas1995np}.
We describe these in detail in \MyAlg{alg:main_loop}.
Note that we depart from the approach of \cite{GrosLasso} who use an active set algorithm. Their approach does not indeed allow warm restarts, which is crucial in our alternating optimization scheme.

\paragraph{Update of $(\etaG)_{G\in\G}$}
The update of $(\etaG)_{G\in\G}$ is straightforward (even if the underlying minimization problem is non-convex), since the minimizer $(\etaG)^*$ in \MyLem{lem:EtaTrick} is given in closed-form. 
In practice, as in \cite{micchelli2006lkf}, we avoid numerical instabilities near zero with the smoothed update $\etaG_k \leftarrow (\etaG_k)^* + \varepsilon$, with $\varepsilon \ll 1$.

\paragraph{Update of $U$}
The update of $U$ follows the technique suggested by \cite{JulienOnLine}. Each column $U^k$ of $U$ is constrained separately through $\Omega_u(U^k)$. Furthermore, if we assume that $V$ and $\{U^j\}_{j\neq k}$ are fixed, some basic algebra leads to
\begin{eqnarray}\label{eq:Uprojection}
\argmin_{ \Omega_u(U^k) \leq 1 } \frac{1}{2np} \NormFro{ X \! - \! U V^\top \! }^2 & = & \argmin_{\Omega_u(U^k) \leq 1 }
					\bigg\|  U^k \! - \! \NormDeux{V^k}^{-2} \!\!
						    (X \! - \! \sum_{j\neq k} [U^j]^\top V^j ) V^k \bigg\|_2^2  \\[-1mm]
						  & = & \argmin_{\Omega_u(U^k) \leq 1 }
					\NormDeux{  U^k \! -  w }^2,
\end{eqnarray}
which is simply the Euclidian projection $\Pi_{\Omega_u}(w)$ of $w$ onto the unit ball of $\Omega_u$. Consequently, the cost of the BCD update of $U$ depends on how fast we can perform this projection; the $\ell_1$ and $\ell_2$ norms are typical cases where the projection can be computed efficiently. In the experiments, we take $\Omega_u$ to be the $\ell_2$ norm.

In addition, since the function $U^k \mapsto \frac{1}{2np} \NormFro{ X \! - \! U V^\top \! }^2$ is continuously differentiable on the (closed convex) unit ball of $\Omega_u$, the convergence of the BCD procedure is guaranteed since the minimum in \MyEq{eq:Uprojection} is unique \cite[Proposition 2.7.1]{bertsekas1995np}. The complete update of $U$ is given in \MyAlg{alg:main_loop}.


\paragraph{Update of $V$}
A fairly natural way to update $V$ would be to compute the closed form solutions available for each row of $V$.
Indeed, both the loss $\frac{1}{2np} \NormFro{ X \! - \! U V^\top \! }^2$ and the penalization on $V$ are separable in the rows of $V$, leading to $p$ independent ridge-regression problems, implying in turn $p$ matrix inversions.

However, in light of the update of $U$, we consider again a BCD scheme on the columns of $V$ that turns out to be much more efficient, without requiring any non-diagonal matrix inversion. The detailed procedure is given in \MyAlg{alg:main_loop}. The convergence follows along the same arguments as those used for $U$.

\begin{algorithm}[!ht]
   \caption{ $\, $ Main optimization procedure for solving \MyEq{eq:main_eq_eta}.}
   \label{alg:main_loop}
	\begin{algorithmic}
	\STATE {\bfseries Input:} Dictionary size $r$, data matrix $X$.
	\STATE {\bfseries Initialization:} Random initialization of $U,V$.
	\STATE {\bfseries $\qquad$ while }( \textit{stopping criterion} not reached )
		\STATE $\qquad \qquad$ \textbf{Update} $(\etaG)_{G \in \G}$: closed-form solution given by \MyLem{lem:EtaTrick}.
		\STATE $\qquad \qquad$ \textbf{Update} $U$ by BCD:
			\STATE {\bfseries $\qquad \qquad \qquad$ for $t=1$ {\bfseries to} $T_u$, $\,$ for $k=1$ {\bfseries to} $r$: }
				\STATE $\qquad \qquad \qquad \qquad
					U^k \leftarrow \Pi_{\Omega_u}(U^k \! + \! \NormDeux{V^k}^{-2} \!\! ( X V^k \! - \! U V^\top V^k ) )$.
		\STATE $\qquad \qquad$ \textbf{Update} $V$ by BCD:
			\STATE {\bfseries $\qquad \qquad \qquad$ for $t=1$ {\bfseries to} $T_v$, $\,$ for $k=1$ {\bfseries to} $r$:}
				\STATE $\qquad \qquad \qquad \qquad$
							    $V^k \!\! \leftarrow \!\!  \Diag{\zeta^k} \Diag{\NormDeux{U^k}^2 \zeta^k \! + \! n p \lambda \mathbf{1}}^{-1} \!\!\!\!\!
				                            ( X^\top U^k \! - V U^\top U^k \! + \NormDeux{U^k}^2 V^k )$.
	\STATE {\bfseries Output:}  Decomposition $U,V$.
	\end{algorithmic}
\end{algorithm}

Our problem is not \textit{jointly} convex in $(\etaG)_{G \in \G}$, $U$ and $V$, which raises the question of the sensitivity of the optimization to its initialization. This point will be discussed in the experiments, \MySec{sec:experiments}.
In practice, the stopping criterion relies on the relative decrease (typically $10^{-3}$) in the cost function in \MyEq{eq:main_eq_new_norm}.

\paragraph{Algorithmic complexity} The complexity of \MyAlg{alg:main_loop} can be decomposed into 3 terms, corresponding to the update procedures of $(\etaG)_{G\in\G},\ U$ and $V$.
We denote by $T_u$ (respectively $T_v$) the number of updates of $U$ (respectively $V$) in \MyAlg{alg:main_loop}.
First, computing $(\etaG)_{G\in\G}$ and $\zeta$ costs
$O( r |\G| + |\G| \sum_{G \in \G}|G| + r \sum_{j=1}^p |G\in\G ; G \ni j|)=O( p r |\G| + p|\G|^2)$.
The update of $U$ requires
$O( (p+T_u n) r^2 + (n p +  C_{\Pi} T_u ) r )$ operations,
where
$C_{\Pi}$
is the cost of projecting onto the unit ball of $\Omega_u$.
Similarly, we get for the update of $V$ a complexity of $O( (n + T_v p) r^2 + n p r)$.
In practice, we notice that the BCD updates for both $U$ and $V$ require only few steps, so that we choose $T_u=T_v=3$.
In our experiments, the algorithmic complexity simplifies to $O( p^2 + r^2 \max\{n, p\} + r p \max\{p^{1/2},n\} )$ times the number of iterations in \MyAlg{alg:main_loop}.

\paragraph{Extension to NMF}
Our formalism does not cover the positivity constraints of non-negative matrix factorization, but it is straightforward to extend it at the cost of an additional threshold operation (to project onto the positive orthant) in the BCD updates of $U$ and $V$.


\section{Experiments}\label{sec:experiments}

We first focus on the application of SSPCA to a face recognition problem and we show that, by adding a sparse structured prior instead of a simple sparse prior, we gain in robustness to occlusions.
We then apply SSPCA to biological data to study the dynamics of a protein/protein complex. 

\begin{figure}[!h]

	\begin{center}
	\begin{tabular}{c}
	\includegraphics[scale=0.9]{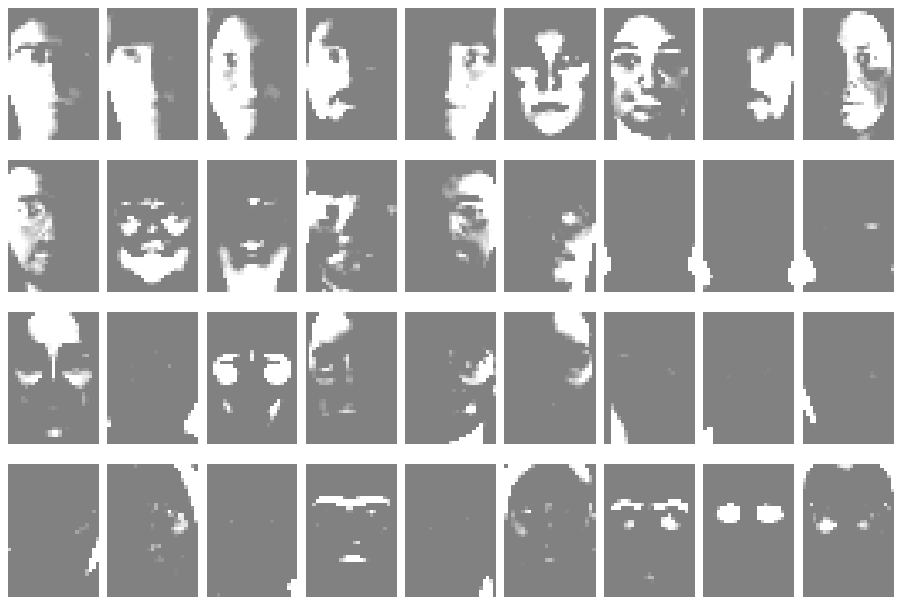}\\
	\\
	\includegraphics[scale=0.9]{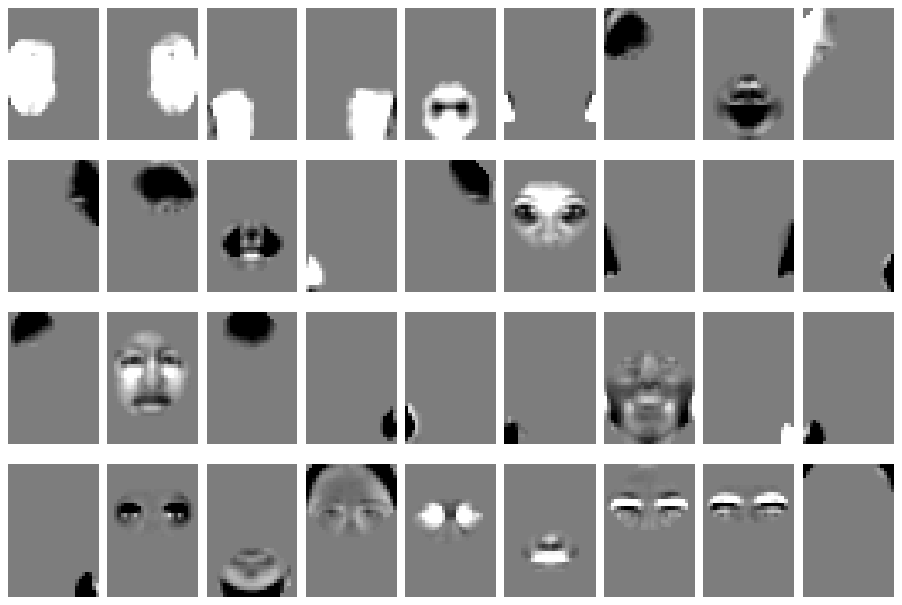}\\
	\\
	\includegraphics[scale=0.9]{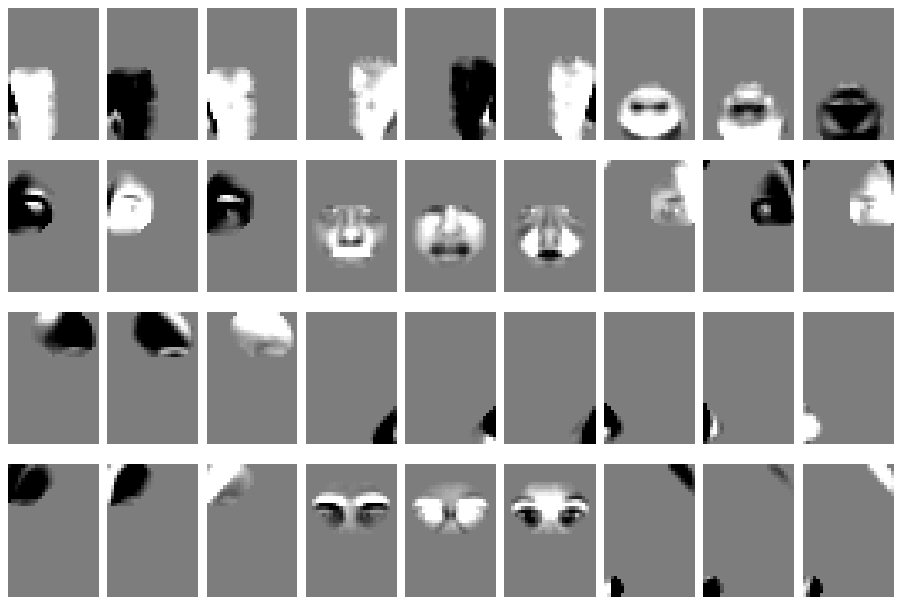}
	\end{tabular}
	\end{center}

\caption{ Three learned dictionaries of faces with $r=36$:  NMF (top), SSPCA (middle) and shared-SSPCA (bottom) (i.e., SSPCA with $|\mathcal{M}|=12$ different patterns of size $3$). The dictionary elements are sorted in decreasing order of variance explained. While NMF gives sparse spatially unconstrained patterns, SSPCA finds convex areas that correspond to more natural face segments. SSPCA captures the left/right illuminations in the dataset by recovering symmetric patterns.}
\label{fig:face_dictionary_examples}
\end{figure}

The results we obtain are validated by known properties of the complex.
In preliminary experiments, we considered the exact regularization of \cite{GrosLasso}, i.e., with $\alpha=1$, but found that the obtained patterns were not sufficiently sparse and salient.
We therefore turned to the setting where the parameter $\alpha$ is in $(0,1)$. In the experiments described in this section we chose $\alpha=0.5$.

\subsection{Face recognition}

We first apply SSPCA on the cropped AR Face Database \cite{ARdataset} that consists of 2600 face images, corresponding to 100 individuals (50 women and 50 men). For each subject, there are 14 non-occluded poses and 12 occluded ones (the occlusions are due to sunglasses and scarfs). We reduce the resolution of the images from 165x120 to 38x27 for computational reasons.

\MyFig{fig:face_dictionary_examples} shows examples of learned dictionaries (for $r=36$ elements), for NMF, SSPCA and SSPCA with shared structure. While NMF finds sparse but spatially unconstrained patterns, SSPCA select sparse convex areas that correspond to a more natural segment of faces. For instance, meaningful parts such as the mouth and the eyes are recovered by the dictionary.

\begin{figure}[!h]
	\begin{center}
        \includegraphics[scale=0.55]{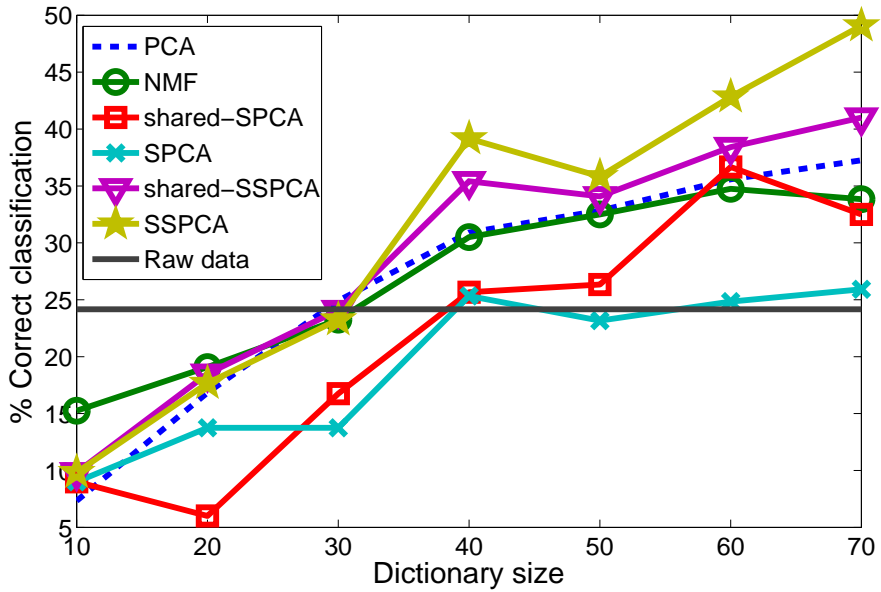}
	\end{center}

\caption{Correct classification rate vs. dictionary size:  each dimensionality reduction technique is used with k-NN to classify occluded faces. SSPCA shows better robustness to occlusions.}
\label{fig:classification_scores}
\end{figure}

We now compare SSPCA, SPCA (as in \cite{lee}), PCA and NMF on a face recognition problem. We first split the data into 2 parts, the occluded faces and non-occluded ones.
For different sizes of the dictionary,
we apply each of the aforementioned dimensionality reduction techniques to the non-occluded faces. Keeping the learned dictionary $V$, we decompose both non-occluded and occluded faces on $V$. We then classify the occluded faces with a k-nearest-neighbors classifier (k-NN), based on the obtained low-dimensional representations.
Given the size of the dictionary, we choose the number of nearest neighbor(s) and the amount of regularization $\lambda$ by 5-fold cross-validation\footnote{In the 5-fold cross-validation, the number of nearest neighbor(s) is searched in $\{1,3,5\}$ while $\log_2(\lambda)$ is in $\{4,6,8,\dots,18\}$. For the dictionary, we consider the sizes $r \in \{10, 20, 30, 40, 50, 60, 70\}$.}.

The formulations of NMF, SPCA and SSPCA are non-convex and as a consequence, the local minima reached by those methods are sensitive to the initialization.
Thus, after having selected the parameters by cross-validation, we run each algorithm 20 times with different initializations on the non-occluded faces, divided into a training (900 instances) and validation set (500 instances) and take the model with the best classification score. We summarize the results in \MyFig{fig:classification_scores}. We denote by shared-SSPCA (resp. shared-SPCA) the models where we impose, on top of the structure of $\Omega^{\alpha}$, to have only 10 different nonzero patterns among the learned dictionaries (see \MySec{sec:shared_structure}).

As a baseline, we also plot the classification score that we obtain when we directly apply k-NN on the raw data, without preprocessing. Because of its local dictionary, SSPCA proves to be more robust to occlusions and therefore outperforms the other methods on this classification task.
On the other hand, SPCA, that yields sparsity without a structure prior, performs poorly.
Sharing structure across the dictionary elements (see \MySec{sec:shared_structure}) seems to help SPCA for which no structure information is otherwise available.

The goal of our paper is not to compete with state-of-the-art techniques of face recognition, but to demonstrate the improvement obtained between $\ell_1$ and more structured norms. We could still improve upon our results using non-linear classification (e.g., with a SVM) or by refining our features (e.g., with a Laplacian filter).


\subsection{Protein complex dynamics}

Understanding the dynamics of protein complexes is important since conformational changes of the complex are responsible for modification of the biological function of the proteins in the complex.
For the EF-CAM complex we consider, it is of particular interest to study the interaction between EF (adenyl cyclase of \textit{Bacillus anthracis}) and CAM (calmodulin) to find new ways to block the action of anthrax \cite{Protein}.

\begin{figure}[!h]

	\begin{center}
	\begin{tabular}{cc}
	\includegraphics[scale=.7]{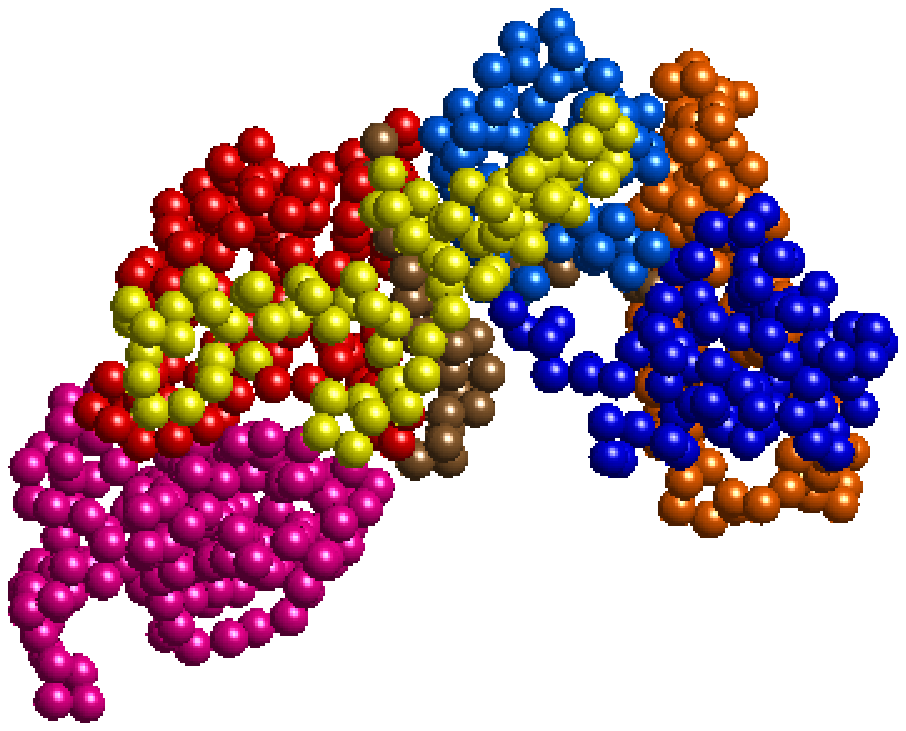} &
        \includegraphics[scale=.7]{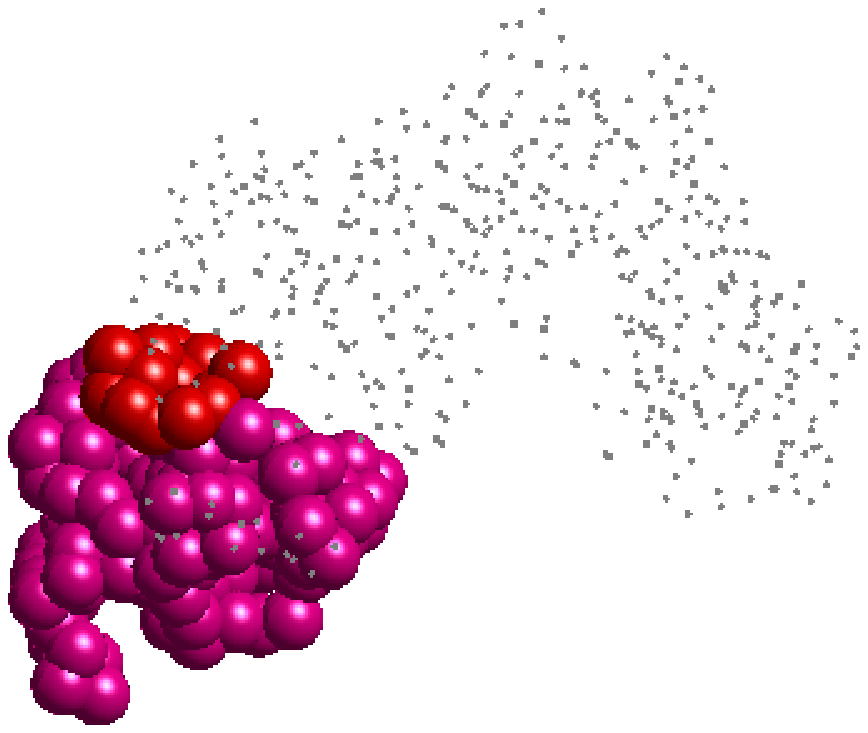} \\
	\includegraphics[scale=.7]{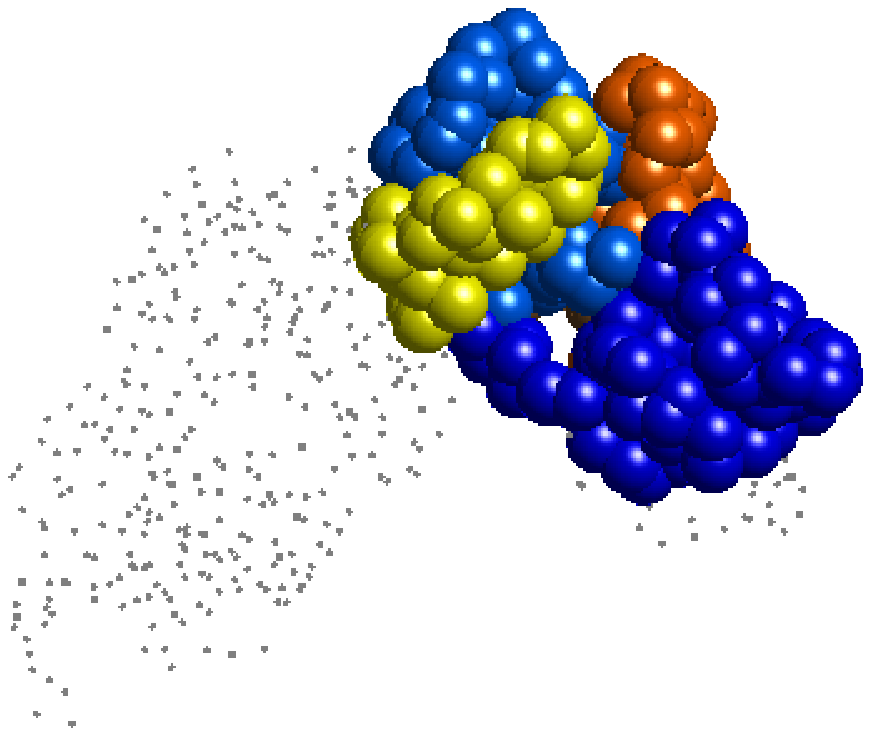} &
	\includegraphics[scale=.7]{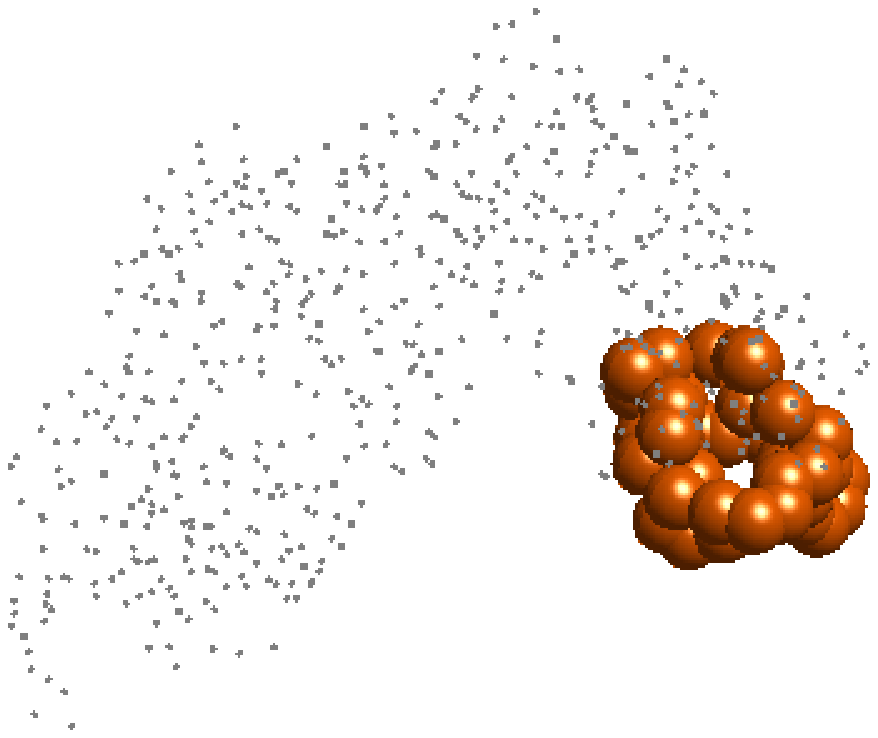}
	\end{tabular}
	\end{center}
	
\caption{{ (Top left) Entire protein with biological domains highlighted in different colors. The two blue parts represent the CAM protein, while the rest of the complex corresponds to EF. (Top right, bottom left/right) Dictionary of size $r=3$ found by SSPCA with the same color code.}}
\label{fig:protein_examples}
\end{figure}

In our experiment, we consider 12000 successive positions of the 619 residues of the EF-CAM complex, each residue being represented by its position in $\mathbb{R}^3$ (i.e., a total of 1857 variables).
We look for dictionary elements that explain the dynamics of the complex, with the constraint that these dictionary elements have to be small convex regions in space. Indeed, the complex is comprised of several functional domains (see \MyFig{fig:protein_examples}) whose spatial structure has to be preserved by our decomposition.

We use the norm $\Omega^{\alpha}$ (and $\G$) to take into account the spatial configuration. Thus, we naturally extend the groups designed for a 2-dimensional grid (see \MyFig{fig:rectangular_group_illustration}) to a 3-dimensional setting.
Since one residue corresponds to 3 variables, we could either
(1) aggregate these variables into a single one and consider a 619-dimensional problem, or
(2) we could use \MySec{sec:shared_structure} to force, for each residue, the decompositions of all three coordinates to share the same support, i.e., in a 1857-dimensional problem. This second method has given us more satisfactory results.

We only present results on a small dictionary (see \MyFig{fig:protein_examples} with $r=3$). As a heuristic to pick $\lambda$, we maximize
{\footnotesize $|\bigcup_{k=1}^r \! \Support{V^k}|^2/ ( p \sum_{k=1}^r \! |\Support{V^k}| )$}
to select dictionary elements that cover pretty well the complex, without too many overlapping areas.

We retrieve groups of residues that match known energetically stable substructures of the complex \cite{Protein} .
In particular, we recover the two tails of the complex and the interface between EF and CAM where the two proteins bind.
Finally, we also run our method on the same EF-CAM complex perturbed by (2 and 4) calcium elements.
Interestingly, we observe stable decompositions, which is in agreement with the analysis of \cite{Protein}.

\vspace*{1cm}


\section{Conclusions}

We proposed to apply a non-convex variant of the regularization introduced by \cite{GrosLasso} to the problem of structured sparse dictionary learning.
We present an efficient block-coordinate descent algorithm with closed-form updates.
For face recognition, the dictionaries learned have increased robustness to occlusions compared to NMF.
An application to the analysis of protein complexes reveals biologically meaningful structures of the complex.
As future directions, we plan to refine our optimization scheme to better exploit sparsity.
We also intend to apply this structured sparsity-inducing norm for multi-task learning, in order to take advantage of the structure between tasks.


\section*{Acknowledgments}

We would like to thank Elodie Laine, Arnaud Blondel and Th\'er\`ese Malliavin from the Unit\'e de Bioinformatique
Structurale, URA CNRS 2185 at Institut Pasteur, Paris for proposing the analysis of the EF-CAM complex and fruitful discussions on protein complex dynamics.
We also thank Julien Mairal for sharing his insights on dictionary learning.


\bibliographystyle{unsrt}
\bibliography{StructuredSparsePCA}

\end{document}